\documentclass[journal]{IEEEtran}
\usepackage[ruled,vlined]{algorithm2e}
\SetAlgoSkip{6pt}      
\SetAlgoInsideSkip{1pt}

\usepackage{algcompatible}
\usepackage{algpseudocode}
\algnewcommand\algorithmicinput{\textbf{Input:}}
\algnewcommand\INPUT{\item[\algorithmicinput]}
\algrenewcommand\algorithmicindent{1.2em} 
\usepackage{cite}
\usepackage[subtle]{savetrees}
\usepackage{verbatim}
\usepackage{subcaption}
\usepackage{graphicx}
\usepackage{siunitx}
\usepackage{amsmath,amssymb,amsfonts,bm}
\usepackage{epsf}
\usepackage{tikz}
\usepackage{psfrag}
\usepackage{pstool}
\usepackage{epstopdf}
\usepackage{bbm} 
\usepackage{booktabs}
\usepackage{stfloats}
\usepackage{url}
\usepackage[dvips]{epsfig}
\usepackage{amsthm}
\usepackage[usenames,dvipsnames]{pstricks}
\usepackage{lettrine}
\usepackage[dvips]{epsfig}
\usepackage{pst-grad} 
\usepackage{pst-plot} 
\usepackage{epsfig}
\usepackage{enumerate}
\usepackage{amsbsy}
\usepackage{amssymb}
\usepackage{amsthm}
\usepackage{amscd}
\usepackage{float}
\usepackage[caption = false]{subfig}
\usepackage[numbers,sort&compress,square]{natbib}

\usepackage{listings}
\usepackage{stackrel}
\usepackage{authblk}
\usepackage{dsfont}
\usepackage[normalem]{ulem}
\makeatletter
\newcommand{\removelatexerror}{\let\@latex@error\@gobble}
\makeatother
\newtheorem{thm}{Theorem}

\newtheorem{prop}{Proposition}

\usepackage{stackengine}

\definecolor{mintbg}{rgb}{.63,.79,.95}
\usepackage{enumitem}
\usepackage{xcolor}

\begin{document}
%
\title{{FedCritic: Serverless Federated Critic Learning-based Resource Allocation for Multi-Cell OFDMA in 6G}}

\author{Amin Farajzadeh, \textit{Member}, \textit{IEEE}, Melike Erol-Kantarci, \textit{Fellow}, \textit{IEEE}\\
\vspace{-3mm}\textit{School of Electrical Engineering and Computer Science, University of Ottawa, Ottawa, Canada}\\
Emails:\{amin.farajzadeh, melike.erolkantarci\}@uottawa.ca
}
\pagenumbering{gobble}
\makeatletter

\patchcmd{\@maketitle}
  {\addvspace{0.5\baselineskip}\egroup}
  {\addvspace{-0.4\baselineskip}\egroup}
 {}
 {}
\makeatother

\maketitle


\begin{abstract}
In sixth-generation (6G) ultra-dense networks, aggressive frequency reuse amplifies inter-cell interference (ICI), making multi-cell orthogonal frequency-division multiple access (OFDMA) scheduling and power control strongly coupled across neighboring cells. We study distributed downlink resource management---joint subcarrier scheduling and power allocation---under interference coupling and long-term per-user quality-of-service (QoS) minimum-rate constraints. By using virtual-queue deficit weights to enforce long-term QoS, we develop FedCritic, a serverless federated multi-agent actor-critic framework with decentralized execution. Unlike centralized training with decentralized execution (CTDE) approaches that require centralized critic learning and joint trajectory aggregation, FedCritic federates the critic through lightweight gossip-based parameter averaging over the interference graph, enabling stable value estimation without a central coordinator while keeping policies local. Simulations in an interference-rich reuse-1 setting show that FedCritic improves mean signal-to-interference-plus-noise ratio (SINR) and cell-edge rate, increases network-wide average sum-rate and fairness relative to non-coordinated and CTDE baselines, and achieves more stable training with lower coordination overhead.
\end{abstract}


%
\IEEEpeerreviewmaketitle

\vspace{-2mm}
\section{Introduction}
\label{sec:introduction}

Sixth-generation (6G) networks in the IMT-2030 timeframe target higher area capacity, enhanced uniform user experience, and stricter latency/reliability under heterogeneous spectrum and ultra-dense deployments \cite{6G_IMT}. These objectives are expected to be pursued via additional spectrum (including upper-midband/cmWave) and denser reuse-1 operation, which strengthens cross-cell coupling and makes inter-cell interference (ICI) a primary limiter of cell-edge throughput and service consistency~\cite{6G_intro}.
\enlargethispage{-2.3\baselineskip}

Orthogonal frequency-division multiple access (OFDMA) is expected to remain central in wideband radio access networks (RANs) due to its fine time--frequency granularity and compatibility with link adaptation and multi-antenna transmission~\cite{OFDMA_ICI_6G}. In reuse-1 multi-cell operation, however, per-cell signal-to-interference-plus-noise ratio (SINR) and rate depend on simultaneous neighbor transmissions, turning ICI mitigation into a network-coupled control problem in the ultra-dense regime~\cite{dens-ICI}. Among the available interference-management levers, subcarrier-level scheduling and power allocation are particularly effective in reuse-1 OFDMA systems because they directly control which links share each subcarrier and at what transmit power. This enables selective protection of dominant victim–interferer pairs at fine time–frequency granularity while still preserving the spectral-efficiency benefits of aggressive frequency reuse~\cite{RRM-ICI-intro}.

Conventional coordination becomes difficult in interference-dominant multi-cell OFDMA~\cite{RRM-ICI-intro}. The problem is high-dimensional, mixes discrete scheduling with continuous power control, and often relies on timely global information and reliable backhaul. This motivates learning-based control. Multi-agent reinforcement learning (MARL) is well aligned with multi-cell resource management because it can optimize long-term objectives under dynamic interference coupling~\cite{MARL-survey}. Many recent works adopt centralized training and decentralized execution (CTDE) actor--critic learning, where a centralized critic leverages joint information to stabilize training while execution remains local at each BS~\cite{CTDE}, e.g., proximal policy optimization (PPO)/MAPPO-style cooperative training frameworks~\cite{MAPPO-WCNC}. However, CTDE typically requires joint trajectory collection and broad information sharing, which limits scalability as the network size and action dimension grow~\cite{resource-MARL-1}.

Recent work has also explored federated RL and MARL for wireless resource management~\cite{survey-FRL}. On the federated RL side, prior studies considered server-based aggregation~\cite{Fed-DRL} and peer-to-peer exchange for decentralized time--frequency control~\cite{decentralized-FRL-heterogeneous}. On the MARL side, related works investigated joint resource-block allocation and power control in multi-cell systems~\cite{MADRL-lit}, dynamic resource management~\cite{resource-MARL}, cooperative user association~\cite{MAPPO-assoc}, and scheduling policies~\cite{MARL-scheduling}. However, scalable and coordination-efficient learning for interference-rich multi-cell OFDMA with joint subcarrier scheduling and power control remains largely unaddressed.
\enlargethispage{-2.3\baselineskip}

We address this gap through a fully decentralized federated MARL framework tailored to reuse-1 OFDMA. Each BS learns from local experience, while only critic parameters are exchanged and averaged among neighboring BSs through gossip over the interference graph. This removes the centralized critic bottleneck of CTDE while preserving critic cooperation and local execution. To the best of our knowledge, this is among the first works to jointly study subcarrier scheduling, discrete power control, long-term per-UE minimum-rate QoS via virtual queues, and serverless neighbor-only critic federation in interference-rich reuse-1 multi-cell OFDMA.

Simulations in an interference-rich multi-cell reuse-1 setting show that the proposed approach improves SINR and sum-rate distributions while achieving higher network-wide average sum-rate and fairness than CTDE baselines and non-learning heuristics, with faster and more stable training under substantially lower coordination overhead.

The remainder of this paper is organized as follows. Section~\ref{sec:system_model} presents the system model and problem formulation. Section~\ref{sec:learning_framework} describes the proposed learning framework. Section~\ref{sec:sim_results} reports simulation results, and Section~\ref{conclusion} concludes the paper.

\section{System Model}\label{sec:system_model}
\subsection{Network Setting and Channel Modeling}

We consider the downlink of a multi-cell OFDMA network composed of $N$ BSs sharing a total bandwidth $B$, partitioned into $K$ orthogonal subcarriers of equal width $\Delta f = B/K$. Each BS $n\in\{1,\dots,N\}$ serves a set of user equipments (UEs) $\mathcal{M}_n$. Time is slotted with index $t=0,1,2,\ldots$, and in each slot $t$, scheduling and power allocation decisions are made online.

Within each slot, BS $n$ schedules at most one UE on each subcarrier $k$ using binary variables $x_{n,k,m}(t)\in\{0,1\}$, which satisfy the intra-cell OFDMA constraint
$\sum_{m\in\mathcal{M}_n} x_{n,k,m}(t) \le 1$.
The transmit power allocated by BS $n$ on subcarrier $k$ in slot $t$ is denoted by $p_{n,k}(t)\ge 0$ and is constrained by the per-BS power budget $\sum_{k=1}^{K} p_{n,k}(t)\le P_n$.

Let $h_{n,k,m}(t)$ denote the small-scale fading coefficient from BS $n$ to UE $m\in\mathcal{M}_n$ on subcarrier $k$ at slot $t$, and let $\alpha_{n,k,m}>0$ denote the large-scale channel gain capturing path loss and log-normal shadowing. We model $\alpha_{n,k,m}=\exp(z_{n,k,m})$, where $z_{n,k,m}\sim\mathcal{N}(\mu_{\mathrm{PL}},\sigma_{\mathrm{PL}}^{2})$, and $\mu_{\mathrm{PL}}$ and $\sigma_{\mathrm{PL}}$ denote the mean and standard deviation of the log-domain large-scale coefficient. The resulting channel power gain is
\vspace{-1mm}
\begin{equation}\label{eq:channel}
g_{n,k,m}(t) = \alpha_{n,k,m}\, |h_{n,k,m}(t)|^2 .
\end{equation}
The small-scale fading is modeled as a first-order complex Gauss--Markov process~\cite{Fed-DRL},
\vspace{-2mm}
\begin{equation}\label{eq:jakes}
h_{n,k,m}(t) = \rho\, h_{n,k,m}(t-1) + \sqrt{1-\rho^2}\, w_{n,k,m}(t),
\end{equation}
where $w_{n,k,m}(t)\sim\mathcal{CN}(0,1)$ is i.i.d.\ across $(n,k,m,t)$ and $0\le\rho<1$ is the temporal correlation coefficient, set according to the Doppler frequency and slot duration. The channel is assumed constant within each slot and evolves across slots according to \eqref{eq:jakes}.
\enlargethispage{-2.3\baselineskip}

For notational convenience, we define the subcarrier activity indicator
$
a_{n,k}(t) \triangleq \sum_{m\in\mathcal{M}_n} x_{n,k,m}(t) \in \{0,1\},
$
which indicates whether BS $n$ transmits on subcarrier $k$ in slot $t$. The downlink SINR experienced by UE $m$ on subcarrier $k$ in slot $t$ is
\begin{equation}\label{eq:sinr}
\gamma_{n,k,m}(t)
=
\frac{p_{n,k}(t)\, g_{n,k,m}(t)}
{\displaystyle \sum_{n'\neq n} a_{n',k}(t)\, p_{n',k}(t)\, g_{n',k,m}(t) + N_0 \Delta f},
\vspace{-1mm}
\end{equation}
where $N_0$ denotes the one-sided noise power spectral density (PSD). For $n'\neq n$, $g_{n',k,m}(t)$ denotes the cross-link power gain from interfering BS $n'$ to UE $m\in\mathcal{M}_n$ on subcarrier $k$.

The instantaneous rate served by BS $n$ on subcarrier $k$ is
\vspace{-1mm}
\begin{equation}
c_{n,k}(t) = \sum_{m\in\mathcal{M}_n} x_{n,k,m}(t)\, \Delta f \log_2\!\bigl(1+\gamma_{n,k,m}(t)\bigr),
\vspace{-1mm}
\end{equation}
and the per-UE and per-BS transmission rates in slot $t$ are, respectively,
\vspace{-3mm}
\begin{align}
R_{n,m}(t) &= \sum_{k=1}^{K} x_{n,k,m}(t)\, \Delta f \log_2\!\bigl(1+\gamma_{n,k,m}(t)\bigr), \label{rate-UE}\\
c_n(t) &= \sum_{k=1}^{K} c_{n,k}(t).
\end{align}

\subsection{Optimization Problem Formulation}

In each slot $t$, the controller at each BS jointly selects the subcarrier--UE scheduling and the per-subcarrier transmit powers to maximize the instantaneous network-wide downlink sum-rate, while satisfying OFDMA orthogonality and per-BS power budgets. The per-UE QoS requirement is defined as a long-term average minimum rate constraint and therefore cannot be enforced directly within a single-slot optimization. Instead, it is handled by an online virtual-queue control mechanism that converts the long-term constraint into time-varying per-slot weights. This yields the following per-slot surrogate optimization problem:
\vspace{-2mm}
\begin{subequations}\label{prob:sumrate_slot}
\begin{align}
\hspace*{-4.5mm} \max_{{\small \mathbf{X}(t),\,\mathbf{P}(t)}} 
& {\small\sum_{n=1}^{N}\sum_{k=1}^{K} c_{n,k}(t)
-\sum_{n=1}^{N}\sum_{m\in\mathcal{M}_n} Q_{n,m}(t)\,\big(R^{\min}_{n,m}-R_{n,m}(t)\big)}
\label{prob:obj_slot}\\
\text{s.t.}\hphantom{0}
& \sum_{m\in\mathcal{M}_n} x_{n,k,m}(t) \le 1, 
\qquad \hphantom{10000000000} \forall n, k, \label{cons:ofdma_slot} \\
& \sum_{k=1}^{K} p_{n,k}(t) \le P_n,\hphantom{0} p_{n,k}(t)\ge 0, 
\qquad \hphantom{00000}\forall n, \label{cons:power_slot} \\
& 0 \le p_{n,k}(t) \le P_n\, a_{n,k}(t),
\qquad \hphantom{00000000} \forall n,k, \label{cons:power-activity_slot}\\
& x_{n,k,m}(t)\in\{0,1\}, 
\qquad \hphantom{00000}\forall n,k,\; m\in\mathcal{M}_n, \label{cons:binary_slot}
\end{align}
\end{subequations}
where $\mathbf{X}(t)=\{x_{n,k,m}(t)\}$ denotes the binary scheduling variables, and $\mathbf{P}(t)=\{p_{n,k}(t)\}$ denotes the per-subcarrier transmit powers.
Constraint~\eqref{cons:ofdma_slot} enforces intra-cell OFDMA orthogonality, \eqref{cons:power_slot} imposes the per-BS power budget, and \eqref{cons:power-activity_slot} couples power allocation to scheduling decisions to ensure that power is used only on active (scheduled) subcarriers. 

The coefficients $\{Q_{n,m}(t)\}$ in \eqref{prob:obj_slot} are virtual queues that quantify nonnegative QoS deficit pressure and evolve over time based on past rate shortfalls. The update is
\enlargethispage{-2.3\baselineskip}
\begin{equation}\label{eq:vq_update}
Q_{n,m}(t{+}1)=\Big[\,Q_{n,m}(t)+R^{\min}_{n,m}-R_{n,m}(t)\,\Big]^+,\:\: \forall n,\; m\in\mathcal{M}_n,
\vspace{-2mm}
\end{equation}
where $[\cdot]^+ \triangleq \max\{\cdot,0\}$. This increases $Q_{n,m}(t)$ when UE $m$ falls below $R^{\min}_{n,m}$ and decreases it otherwise. Since $\sum_{n,m} Q_{n,m}(t)R^{\min}_{n,m}$ is constant at slot $t$, \eqref{prob:obj_slot} equivalently maximizes a weighted served-rate objective that prioritizes UEs with larger deficits. As a result, repeatedly solving \eqref{prob:sumrate_slot} online steers the system toward satisfying long-term average QoS while still prioritizing instantaneous sum-rate.

Problem~\eqref{prob:sumrate_slot} is a mixed-integer nonconvex program due to binary scheduling decisions and interference-coupled SINR expressions under reuse-1, making centralized real-time optimization impractical at scale. This motivates an online decentralized learning formulation in which each BS learns from local observations and coordinates only with neighboring BSs to account for interference coupling, without relying on a central controller.


\vspace{-1mm}
\section{Reformulation as an Interference--Coupled Dec--POMDP}
\label{sec:learning_framework}

We model the online per-slot control problem in \eqref{prob:sumrate_slot} as an interference-coupled Dec--POMDP with decentralized execution at the BSs. The model captures partial observability, since each BS has access only to local measurements and lightweight neighbor summaries, and interference coupling, since achieved rates depend on other BSs' actions through \eqref{eq:sinr}. Following the cooperative Dec--POMDP, all BSs optimize a shared team reward that matches the network-wide surrogate objective in \eqref{prob:obj_slot}. Based on this formulation, we develop an on-policy actor--critic method and introduce FedCritic, a serverless federated critic mechanism that stabilizes value estimation via periodic gossip mixing over the interference graph.
\vspace{-2mm}
\subsection{Interference--Coupled Dec--POMDP Model}

We model the network as the Dec--POMDP tuple
\vspace{-1mm}
\begin{equation}
\label{eq:decpomdp_tuple_final}
\mathfrak{D} \triangleq
\big\langle
\mathcal{S},\{\mathcal{A}_n\}_{n=1}^{N},\mathcal{P},\{\mathcal{O}_n\}_{n=1}^{N},
\{Z_n\}_{n=1}^{N}, r,\beta
\big\rangle,
\vspace{-1mm}
\end{equation}
where $\beta\in(0,1]$ is the discount factor. The global state $s(t)\in\mathcal{S}$ collects the channel variables governing the network dynamics, including the fading coefficients $\{h_{n,k,m}(t)\}$ evolving via \eqref{eq:jakes} and the large-scale gains $\{\alpha_{n,k,m}\}$. 
To preserve Markovian dynamics under long-term constraints and neighbor summaries, we include in $s(t)$ the virtual queues $\{Q_{n,m}(t)\}$ and the exponential moving average (EMA) statistics $\{\widehat{O}_{n,k}(t)\}$ used by the controller.
Each BS $n$ selects a local action $a_n(t)\in\mathcal{A}_n$, and the joint action $a(t)=(a_1(t),\dots,a_N(t))$ induces the next state through the Markov transition kernel
$
\mathcal{P}\big(s(t{+}1)\mid s(t),a(t)\big),
$
which captures the channel evolution. 
BS $n$ does not observe $s(t)$ directly; instead it receives a local observation $o_n(t)\in\mathcal{O}_n$ via a deterministic observation function $Z_n:\mathcal{S}\to\mathcal{O}_n$, i.e., $o_n(t)=Z_n(s(t))$, consisting of locally available features augmented with low-rate neighbor summaries exchanged over an interference graph. All BSs receive the same cooperative team reward $r(t)=r\big(s(t),a(t)\big)$, implemented by the network-wide shaped utility defined in Section~III-C.
\enlargethispage{-2.3\baselineskip}

\subsubsection{Interference graph and neighbor summaries}
Let the interference graph be $\mathcal{G}_I=(\{1,\dots,N\},\mathcal{E}_I)$, where $(n,j)\in\mathcal{E}_I$ indicates that BS $j$ is a dominant interferer to cell $n$, and let $\mathcal{N}(n)$ denote the neighbor set (assume $|\mathcal{N}(n)|\ge 1$; otherwise set $\overline{O}_{n,k}(t)=O^{\max}_{n,k}(t)=0$). To provide a low-overhead proxy of interference pressure, BSs exchange subcarrier-level activity statistics. Let $\widehat{O}_{n,k}(t)$ be an EMA of the activity indicator $a_{n,k}(t)$ as
\vspace{-2mm}
\begin{equation}
\label{eq:occ_ema_final}
\widehat{O}_{n,k}(t{+}1)=\alpha_O\,\widehat{O}_{n,k}(t)+(1-\alpha_O)\,a_{n,k}(t),
\end{equation}
\vspace{-2mm}
with $\alpha_O\in(0,1)$. BS $n$ forms neighbor aggregates
\begin{align}
\label{eq:neighbor_occ_final}
\overline{O}_{n,k}(t)&=\frac{1}{|\mathcal{N}(n)|}\sum_{j\in\mathcal{N}(n)}\widehat{O}_{j,k}(t),
\\
O^{\max}_{n,k}(t)&=\max_{j\in\mathcal{N}(n)}\widehat{O}_{j,k}(t),
\end{align}
which are included in $o_n(t)$.

\subsubsection{Actions and feasibility mapping}

Each BS $n$ selects actions that determine its local components of $(\mathbf{X}(t),\mathbf{P}(t))$. For each subcarrier $k$, the action is parameterized as
\vspace{-2mm}
\begin{equation}
\label{eq:action_def_final}
\tilde a_{n,k}(t)=\big(u_{n,k}(t),\,m_{n,k}(t),\,\ell_{n,k}(t)\big),
\end{equation}
where $u_{n,k}(t)\in\{0,1\}$ indicates mute/active transmission, $m_{n,k}(t)\in\mathcal{M}_n$ is the scheduled UE when $u_{n,k}(t)=1$, and $\ell_{n,k}(t)\in\{1,\dots,L\}$ selects a discrete power level $p^{(\ell)}$. Note that $u_{n,k}(t)$ coincides with the activity indicator $a_{n,k}(t)$. The induced optimization variables are
\vspace{-3mm}
\begin{align}
x_{n,k,m}(t) &= \mathbbm{1}\{u_{n,k}(t)=1,\; m_{n,k}(t)=m\},\\
p_{n,k}(t) &= u_{n,k}(t)\,p^{(\ell_{n,k}(t))}.
\end{align}
By construction, $\sum_{m\in\mathcal{M}_n}x_{n,k,m}(t)=u_{n,k}(t)$ and the intra-cell OFDMA constraint \eqref{cons:ofdma_slot} is satisfied. To enforce \eqref{cons:power_slot}--\eqref{cons:power-activity_slot}, we apply masking (forcing $p_{n,k}(t)=0$ when muted) and a per-slot power normalization (projection) to satisfy the sum-power constraint
\begin{equation}
\label{eq:power_scale_final}
p_{n,k}(t)\leftarrow p_{n,k}(t)\cdot \min\Big\{1,\frac{P_n}{\sum_{k'=1}^{K}p_{n,k'}(t)+\epsilon}\Big\},\quad \forall k,
\end{equation}
where $\epsilon>0$ is a small constant for numerical stability. This normalization preserves relative per-subcarrier levels while ensuring $\sum_k p_{n,k}(t)\le P_n$.
\vspace{-1mm}
\enlargethispage{-2.3\baselineskip}
\subsection{Reward Design and Long-Term QoS Enforcement}

To mirror the per-slot objective \eqref{prob:obj_slot}, we define each BS's base transmission rate reward as
\begin{equation}
\label{eq:reward_rate_final}
r^{\mathrm{rate}}_n(t)=c_n(t)=\sum_{k=1}^{K} c_{n,k}(t).
\end{equation}
Using the per-UE rate definition in~\eqref{rate-UE}, we incorporate QoS pressure via the deficit weights $\{Q_{n,m}(t)\}$ through
\begin{equation}
\label{eq:qos_pressure_term}
r^{\mathrm{qos}}_n(t)= \sum_{m\in\mathcal{M}_n} Q_{n,m}(t)\,R_{n,m}(t),
\end{equation}
where constant terms independent of the slot-$t$ decision are omitted.

Furthermore, to promote interference-aware operation, we include a leakage-style penalty based on long-term cross-link gain proxies. Let $\bar g_{n\to j,k}$ denote the long-term average cross-link gain from BS $n$ to cell $j$ on subcarrier $k$, obtained by time-averaging the corresponding measured cross-link gains. Define
\vspace{-2mm}
\begin{equation}
\label{eq:leakage_def_final}
L_{n}(t)
=\sum_{k=1}^{K}\sum_{j\in\mathcal{N}(n)} \eta_{nj}\, p_{n,k}(t)\,\bar g_{n\to j,k},
\end{equation}
where $\eta_{nj}\ge 0$ are fixed leakage weights. The shaped per-slot reward is
\vspace{-2mm}
\begin{equation}
\label{eq:reward_shaped_final}
\tilde r_n(t)= r^{\mathrm{rate}}_n(t) + r^{\mathrm{qos}}_n(t) - \lambda_{\mathrm{int}}\, L_{n}(t),
\end{equation}
with $\lambda_{\mathrm{int}}\ge 0$. In the cooperative Dec--POMDP, we use the team reward
\vspace{-3mm}
\begin{equation}
\label{eq:team_reward_def}
r(t)\triangleq \sum_{n=1}^{N}\tilde r_n(t),
\end{equation}
and the environment provides the same scalar $r(t)$ to all BSs.
%

The deficit weights $\{Q_{n,m}(t)\}$ evolve via the virtual-queue recursion \eqref{eq:vq_update}. Mean-rate stability of these queues implies satisfaction of the long-term QoS constraints in the Lyapunov sense; operationally, large $Q_{n,m}(t)$ increases the urgency to serve UE $m$, consistent with \eqref{prob:obj_slot}.
\vspace{-1mm}
\enlargethispage{-2.3\baselineskip}
\subsection{Federated Critic Learning With Gossip-Based Aggregation (FedCritic)}
\label{subsec:fedcritic}

Each BS $n$ runs an on-policy actor--critic learner conditioned on its local observation $o_n(t)$. The actor $\pi_{\theta_n}$ outputs a distribution over a structured, feasibility-constrained action $a_n(t)$, factorized across subcarriers. For each subcarrier $k$, the BS first decides an activity/muting variable $u_{n,k}(t)\in\{0,1\}$ and, if active, selects a scheduled UE index $m_{n,k}(t)\in\mathcal{M}_n$ and a transmit-power level index $\ell_{n,k}(t)\in\{1,\dots,L\}$. The resulting policy factorization is
\vspace{-1mm}
\begin{align}
\label{eq:policy_factorization_fedcritic}
&\pi_{\theta_n}\!\big(a_n(t)\mid o_n(t)\big)
=\nonumber\\
&\prod_{k=1}^{K}
\pi_{\theta_n}\!\big(u_{n,k}(t)\mid o_n(t)\big)\,
\pi_{\theta_n}\!\big(m_{n,k}(t)\mid o_n(t),u_{n,k}(t){=}1\big)\nonumber\\
&\quad\times
\pi_{\theta_n}\!\big(\ell_{n,k}(t)\mid o_n(t),u_{n,k}(t){=}1\big),
\end{align}
where invalid choices are removed via action masking to satisfy the intra-cell OFDMA constraint, and the per-BS sum-power constraint is enforced via the normalization in \eqref{eq:power_scale_final}.

The local critic $V_{\psi_n}$ estimates the expected team return conditioned on the local information state,
\vspace{-1mm}
\begin{equation}
V_{\psi_n}(o_n(t)) \approx \mathbb{E}\!\left[\sum_{\tau=t}^{t+H-1}\beta^{\tau-t}\, r(\tau)\ \bigg|\ o_n(t)\right],
\vspace{-1mm}
\end{equation}
where $r(t)$ is the team reward in \eqref{eq:team_reward_def} and $H$ is the rollout length. Given local rollout tuples $\{(o_n(t),a_n(t), r(t),o_n(t{+}1))\}$, the critic is trained via temporal-difference (TD) regression
\vspace{-1mm}
\begin{align}
\label{eq:critic_loss_fedcritic}
\mathcal{L}^{V}_n(\psi_n)
&=
\mathbb{E}\!\left[\big(V_{\psi_n}(o_n(t))-\widehat V_n(t)\big)^2\right],\\
\widehat V_n(t)
&\triangleq r(t)+\beta\,V_{\psi_n}\!\big(o_n(t{+}1)\big).
\vspace{-2mm}
\end{align}
The actor parameters $\theta_n$ are updated locally using a PPO surrogate objective with advantage estimates derived from $V_{\psi_n}$. We estimate advantages using generalized advantage estimation (GAE) with parameter $\lambda_{\mathrm{adv}}\in[0,1]$, which controls the bias--variance tradeoff of the advantage estimator \cite{MAPPO-estimator}. All updates above are decentralized and use only local trajectories.
\begin{algorithm}[t]
\caption{FedCritic-MAPPO at BS $n$}
\label{alg:fedcritic_mappo}
\DontPrintSemicolon
\footnotesize

\KwIn{
neighbors $\mathcal{N}(n)$; weights $\{w_{nj}\}$; gossip period $K_g$;
rollout length $H$; PPO epochs $E$; mini-batch size $B_{\mathrm{mb}}$;
$(\beta,\lambda_{\mathrm{adv}})$; power budget $P_n$; power levels $\{p^{(\ell)}\}_{\ell=1}^{L}$;
EMA factor $\alpha_O$.
}
\KwOut{actor params $\theta_n$, critic params $\psi_n$.}

\textbf{Init:}\quad actor $\pi_{\theta_n}$, critic $V_{\psi_n}$; set $\theta_n^{old}\leftarrow \theta_n$\;
\hspace*{3.05em}queues $\{Q_{n,m}\}$; occupancy EMA $\{\widehat O_{n,k}\}$\;

\For{$r=1,2,\ldots$}{
  \textbf{Neighbor info:}\quad receive $\{\widehat O_{j,k}\}_{j\in\mathcal{N}(n)}$; build $o_n(t)$\;
  \textbf{Buffer:}\quad $\mathcal{D}_n\leftarrow\emptyset$\;

  \For{$t=0,\ldots,H-1$}{
    \textbf{Observe:}\quad $o_n(t)$\;
    \textbf{Act:}\quad $a_n(t)\sim\pi_{\theta_n}(\cdot|o_n(t))$ (per-subcarrier mute/UE/power)\;
    \textbf{Project:}\quad map to $(x_{n,k,m}(t),p_{n,k}(t))$; enforce feasibility via masking/normalization\;
    \textbf{Step:}\quad execute; get $R_{n,m}(t)$; compute local $\tilde r_n(t)$ and obtain team reward $r(t)=\sum_{j=1}^{N}\tilde r_j(t)$\;
    \textbf{Update:}\quad $Q_{n,m}\!\leftarrow\![Q_{n,m}+R^{\min}_{n,m}-R_{n,m}(t)]^{+}$\;
    \hspace*{3.15em}$\widehat O_{n,k}\!\leftarrow\!\alpha_O\widehat O_{n,k}+(1-\alpha_O)a_{n,k}(t)$\;
    \textbf{Store:}\quad $(o_n(t),a_n(t), r(t),o_n(t{+}1))\in\mathcal{D}_n$\;
  }

  \textbf{Advantage (GAE):}\quad compute $\widehat A_n(t)$ using $V_{\psi_n}$ and $(\beta,\lambda_{\mathrm{adv}})$\;
  \textbf{PPO:}\quad set $\theta_n^{old}\leftarrow\theta_n$\;
  \For{$e=1,\ldots,E$}{
    \ForEach{mini-batch $\mathcal{B}\subset\mathcal{D}_n$, $|\mathcal{B}|=B_{\mathrm{mb}}$}{
      \textbf{Critic:}\quad TD regression update on $\mathcal{B}$\;
      \textbf{Actor:}\quad PPO clipped update on $\mathcal{B}$ using $\theta_n^{old},\widehat A_n$\;
    }
  }

  \If{$r \bmod K_g = 0$}{
    exchange $\{\psi_j\}_{j\in\mathcal{N}(n)}$;\;
    $\psi_n \leftarrow \sum_{j\in\mathcal{N}(n)\cup\{n\}} w_{nj}\,\psi_j$\;
  }
}
\end{algorithm}

\noindent\textbf{Gossip-based federated critic mixing:}
In interference-coupled multi-cell OFDMA, each BS's reward and transition dynamics depend on the (unobserved) actions of neighboring BSs, which can destabilize purely local critic learning. FedCritic addresses this without a centralized coordinator by federating only the critic over the interference graph $\mathcal{G}_I$ via lightweight gossip.

Let $W=[w_{nj}]\in\mathbb{R}^{N\times N}$ be a doubly-stochastic mixing matrix consistent with $\mathcal{G}_I$, i.e., $w_{nj}>0$ only if $j=n$ or $(n,j)\in\mathcal{E}_I$, and $\sum_{j=1}^{N}w_{nj}=1$, $\sum_{n=1}^{N}w_{nj}=1$. Every $K_g$ learning rounds (i.e., after each policy/critic update), BSs perform a neighbor aggregation step
\vspace{-1mm}
\begin{equation}
\label{eq:gossip_critic_fedcritic}
\psi_n \leftarrow \sum_{j\in\mathcal{N}(n)\cup\{n\}} w_{nj}\,\psi_j,\qquad \forall n,
\end{equation}
implemented by exchanging critic parameters with one-hop neighbors. Over each connected component of $\mathcal{G}_I$, repeated application of \eqref{eq:gossip_critic_fedcritic} drives $\{\psi_n\}$ toward consensus, improving the stability of value estimates under strong interference coupling. Theorem~\ref{thm:critic_consensus} formalizes this effect by establishing consensus (in mean-square) under periodic gossip mixing, and Proposition~\ref{prop:advantage_bias} shows that critic consensus implies vanishing disagreement in TD-based advantage estimates used by PPO. Algorithm~\ref{alg:fedcritic_mappo} summarizes the proposed FedCritic-MAPPO procedure.
\enlargethispage{-2.3\baselineskip}

We federate only the critic to stabilize learning under interference coupling while keeping actor execution local and signaling lightweight.
\vspace{-1mm}
\begin{thm}[Critic consensus under periodic gossip mixing]
\label{thm:critic_consensus}
Let $t$ denote the environment slot index and $s$ the learning round index.
Assume a fixed joint policy $\pi$ during the critic-update window. Each BS $n$ maintains critic parameters
$\psi_n^s\in\mathbb{R}^d$ and performs a local stochastic-gradient step
\vspace{-5mm}
\begin{align}
\psi_n^{s+\frac12} &= \psi_n^s-\eta_s\,g_n(\psi_n^s;\xi_n^s), \label{eq:local_sgd}\\
\mathbb{E}\!\left[g_n(\psi_n^s;\xi_n^s)\mid\mathcal{F}_s\right] &= \nabla F_n(\psi_n^s), \label{eq:unbiased_grad}
\vspace{-2mm}
\end{align}
where $\{\mathcal{F}_s\}$ is the natural filtration, $\xi_n^s$ is the local sample/mini-batch at iteration $s$,
$F_n:\mathbb{R}^d\!\to\!\mathbb{R}$ is the local critic loss, and $\eta_s>0$.

Every $K_g$ learning rounds (i.e., for $s\in\{0,K_g,2K_g,\ldots\}$), critics are mixed via gossip as
\vspace{-4mm}
{\small
\begin{equation}
\psi^{s+1}=(W_s\otimes I_d)\psi^{s+\frac12}, \qquad
W_s=\begin{cases}
W,& s\equiv 0\!\!\!\pmod{K_g},\\
I_N,& \text{otherwise},
\end{cases}
\label{eq:gossip_stack}
\vspace{-2mm}
\end{equation}
}
where $I_d$ and $I_N$ are the $d\times d$ and $N\times N$ identity matrices, respectively, and
$\psi^s=[(\psi_1^s)^\top,\ldots,(\psi_N^s)^\top]^\top$.

Assume $W$ is doubly stochastic and consistent with a connected interference graph, and define
\vspace{-2mm}
\begin{equation}
\sigma \triangleq \left\|W-\frac{1}{N}\mathbf{1}\mathbf{1}^\top\right\|_2<1,
\label{eq:sigma_def}
\end{equation}
where $\mathbf{1}\in\mathbb{R}^N$ is the all-ones vector and $\|\cdot\|_2$ is the spectral norm.
Assume: (A1) each $F_n$ is $L$-smooth and bounded below; (A2)
$\mathbb{E}[\|g_n(\psi_n^s;\xi_n^s)\|^2\mid\mathcal{F}_s]\le G^2$; (A3) $\sum_s \eta_s=\infty$ and $\sum_s \eta_s^2<\infty$.
Let $\bar\psi^s\triangleq \frac1N\sum_{n=1}^N \psi_n^s$ and $F(\psi)\triangleq \frac1N\sum_{n=1}^N F_n(\psi)$. Then
\begin{enumerate}
\item \textbf{Consensus:} $\displaystyle \lim_{s\to\infty}\mathbb{E}\!\left[\sum_{n=1}^N\|\psi_n^s-\bar\psi^s\|^2\right]=0$.
\item \textbf{Stationarity of the average:} $\displaystyle \liminf_{s\to\infty}\mathbb{E}\,\|\nabla F(\bar\psi^s)\|^2=0$.
\end{enumerate}
\end{thm}
\proof See Appendix~\ref{app:proof_thm1}.
%
\enlargethispage{-2.3\baselineskip}
\begin{prop}[Implication for TD advantage estimation in Algorithm~\ref{alg:fedcritic_mappo}]
\label{prop:advantage_bias}
Under Theorem~\ref{thm:critic_consensus}, assume $V_\psi(\cdot)$ is Lipschitz in $\psi$ on $\mathcal{O}$, i.e.,
\vspace{-1mm}
\begin{equation}
\label{eq:V_lipschitz}
\hspace{6mm}|V_{\psi}(o)-V_{\psi'}(o)|\le L_V\|\psi-\psi'\|,\qquad \forall o\in\mathcal{O},
\end{equation}
where $L_V>0$ is a Lipschitz constant. For the rollout collected at learning round $s$ and any slot
$t$, define the one-step TD advantage at BS $n$ as
\vspace{-1mm}
\begin{equation}
\hspace{7mm}\widehat{A}^{(n)}_{t,s}\triangleq r(t)+\beta V_{\psi_n^s}\!\big(o_n(t{+}1)\big)-V_{\psi_n^s}\!\big(o_n(t)\big),
\end{equation}
and define $\widehat{A}^{(\mathrm{avg})}_{t,s}$ analogously using $V_{\bar\psi^s}$. Then
\begin{equation}
\mathbb{E}\big|\widehat{A}^{(n)}_{t,s}-\widehat{A}^{(\mathrm{avg})}_{t,s}\big|
\le (1+\beta)L_V\,\mathbb{E}\|\psi_n^s-\bar\psi^s\|,
\qquad \forall n,s,\ t. 
\end{equation}
Hence, the advantage disagreement vanishes as $\mathbb{E}\|\psi_n^s-\bar\psi^s\|\to 0$.
\end{prop}
Theorem~\ref{thm:critic_consensus} characterizes the critic mixing dynamics and, via Proposition~\ref{prop:advantage_bias}, its impact on advantage estimation.
Global convergence of the full PPO-based MARL loop with nonlinear function approximation is beyond scope.
\subsection{Learning Objective}

Under the shared team reward \eqref{eq:team_reward_def}, the cooperative Dec--POMDP objective is
\vspace{-2mm}
{\small
\begin{equation}
\label{eq:team_objective_final}
\max_{\{\theta_n\}_{n=1}^{N}}\;
J(\theta)
=
\mathbb{E}
\Bigg[
\sum_{t=0}^{H-1}
\beta^t\, r(t)
\Bigg],
\end{equation}
}
where $\theta\triangleq(\theta_1,\ldots,\theta_N)$ and the joint policy factorizes as $\pi_{\theta}(a(t)\mid o(t))=\prod_{n=1}^{N}\pi_{\theta_n}(a_n(t)\mid o_n(t))$ with decentralized execution.
The mapping in \eqref{eq:action_def_final} induces $(\mathbf{X}(t),\mathbf{P}(t))$ and enforces \eqref{cons:ofdma_slot}--\eqref{cons:power-activity_slot} via action masking and power normalization.
Finally, periodic neighbor-only critic mixing in \eqref{eq:gossip_critic_fedcritic} provides serverless coordination over $\mathcal{G}_I$, aiming to improve stability in interference-dominant regimes.
\vspace{-1mm}
\section{Simulation Results}
\label{sec:sim_results}

We evaluate a reuse-1 multi-cell OFDMA downlink with $N=7$ BSs, $K=32$ subcarriers, and $M=8$ UEs per cell. Each BS allocates transmit power under a per-BS budget $P_{\max}=1.0$ using discretized levels $\{0.05,0.15,0.35,0.60,1.0\}$, with noise PSD $10^{-3}$. Large-scale gains follow a log-normal model with $(\mu_{\mathrm{PL}},\sigma_{\mathrm{PL}})=(-2.3,0.8)$ and cross-link scaling $1.2$, while small-scale fading evolves via a Gauss--Markov process with correlation $\rho=0.85$. Inter-cell coupling is defined by a line neighbor graph with radius $1$.
\enlargethispage{-2.3\baselineskip}

We compare four learning methods: B1: CTDE (MAPPO with centralized training and decentralized execution), B2: CTDE+VQ (CTDE augmented with virtual queues), B3: FedActor (federated training with periodic actor mixing across BSs), and Proposed: FedCritic (neighbor-gossiped critic with local execution), along with two non-learning heuristics: GREEDY and QoS. Training uses PPO for $250$ updates with rollout horizon $H=128$; since actions are selected per subcarrier, each BS contributes $H\!\times\!K$ samples per update, making $B_{\mathrm{mb}}=256$ feasible. We set $\beta=0.99$, $\lambda_{\mathrm{adv}}=0.95$, clipping $\epsilon=0.2$, $6$ epochs per update, mini-batch size $256$, and max gradient norm $0.5$, with entropy coefficient decayed from $0.010$ to $0.001$. We set the per-UE minimum-rate target to $R_{\min}=2.0$ and evaluate every $10$ updates over $6$ seeds, $6$ episodes/seed, and $24$ steps/episode; we report mean $\pm$ $95\%$ confidence intervals over the $6$ per-seed averages. We set $K_g=1$ (critic mixing every PPO update) to track strong inter-cell coupling in the interference-dominant reuse-1 regime.
\begin{figure}[t]
\centering
\includegraphics[width=1.01\linewidth]{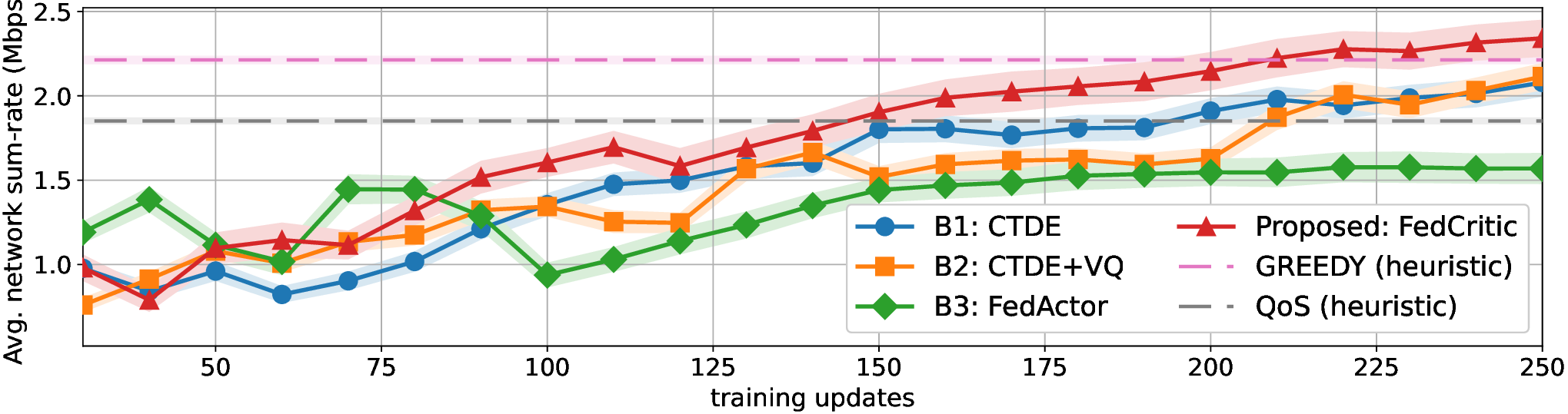}
\caption{Evaluation average sum-rate per-slot versus training updates.}
\label{fig:eval_return}
\end{figure}
\begin{figure}[t]
\centering
\includegraphics[width=0.82\linewidth]{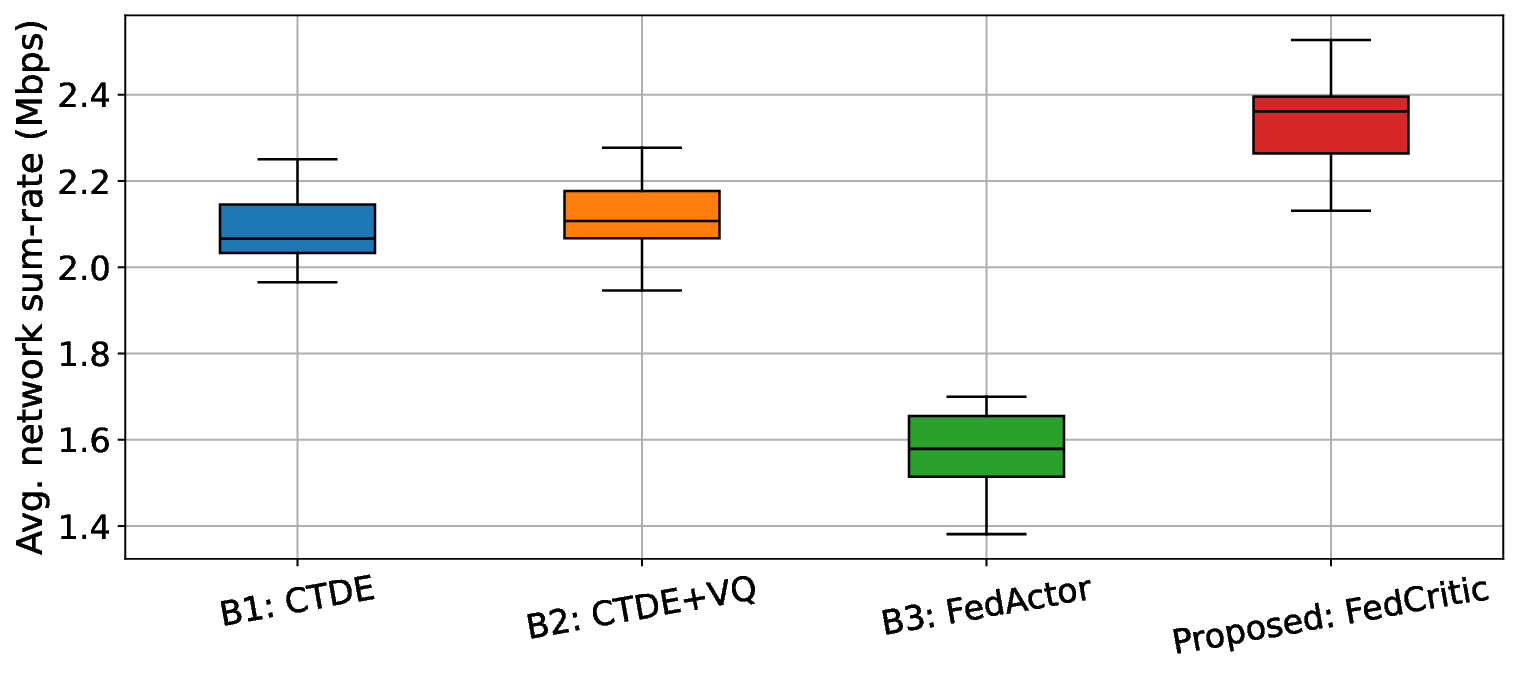}
\caption{Distribution of the per-slot average network sum-rate for the final learned policies.} 
\label{fig:boxplot_return}
\end{figure}
%

\textbf{Learning dynamics and final performance:}
Fig.~\ref{fig:eval_return} shows the per-slot average network sum-rate at evaluation checkpoints. Proposed-FedCritic exhibits the strongest learning dynamics and reaches the highest evaluation sum-rate among the learning-based methods. GREEDY and QoS appear as fixed references: GREEDY is strong but interference-unaware, whereas QoS is more conservative due to its constraint-driven allocation. Fig.~\ref{fig:boxplot_return} reports the distribution of the final per-slot average network sum-rate across evaluation episodes. Proposed-FedCritic achieves the highest median and a tighter spread than the CTDE baselines and B3-FedActor, indicating more reliable performance.

\textbf{Interference robustness:}
Fig.~\ref{fig:ICI} summarizes interference outcomes over active links. In Fig.~\ref{fig:ICI}(a), the CTDE baselines (B1/B2) operate in a substantially lower-SINR regime, whereas B3-FedActor and especially Proposed-FedCritic shift the operating point to higher SINR, reflecting more interference-aware scheduling under reuse-1. Fig.~\ref{fig:ICI}(b) reports the neighbor-collision rate. B1/B2 remain near fully colliding reuse patterns, while B3-FedActor and Proposed-FedCritic drastically reduce neighbor collisions, demonstrating learned coordination in the interference-rich setting.
\enlargethispage{-2.3\baselineskip}

\textbf{Distributional behavior (CDFs):}
Fig.~\ref{fig:CDF}(a) presents the CDF of per-active subcarrier SINR. Proposed-FedCritic produces a clear right-shift, reducing the probability of low-SINR transmissions under reuse-1 interference. Fig.~\ref{fig:CDF}(b) shows the CDF of per-active subcarrier-link rate (Mbps), where Proposed-FedCritic improves the high-rate tail by selectively activating and powering links that remain favorable under neighbor interference.

\textbf{Reuse structure and interpretability:}
Fig.~\ref{fig:activity_heatmap_compare} visualizes the learned reuse intensity as BS-by-subcarrier activity. B2: CTDE+VQ tends toward near-uniform high activity across the grid, consistent with its high neighbor-collision behavior. In contrast, Proposed-FedCritic learns structured spatial-frequency reuse with selective activation, which explains its improved SINR distribution and superior network-wide average sum-rate in Figs.~\ref{fig:eval_return} and~\ref{fig:boxplot_return}. Overall, Proposed-FedCritic achieves a better interference--throughput trade-off in an interference-dominant reuse-1 multi-cell OFDMA regime.

\begin{figure}[t]
\centering
\begin{subfigure}[t]{0.49\linewidth}
    \centering
    \includegraphics[width=\linewidth]{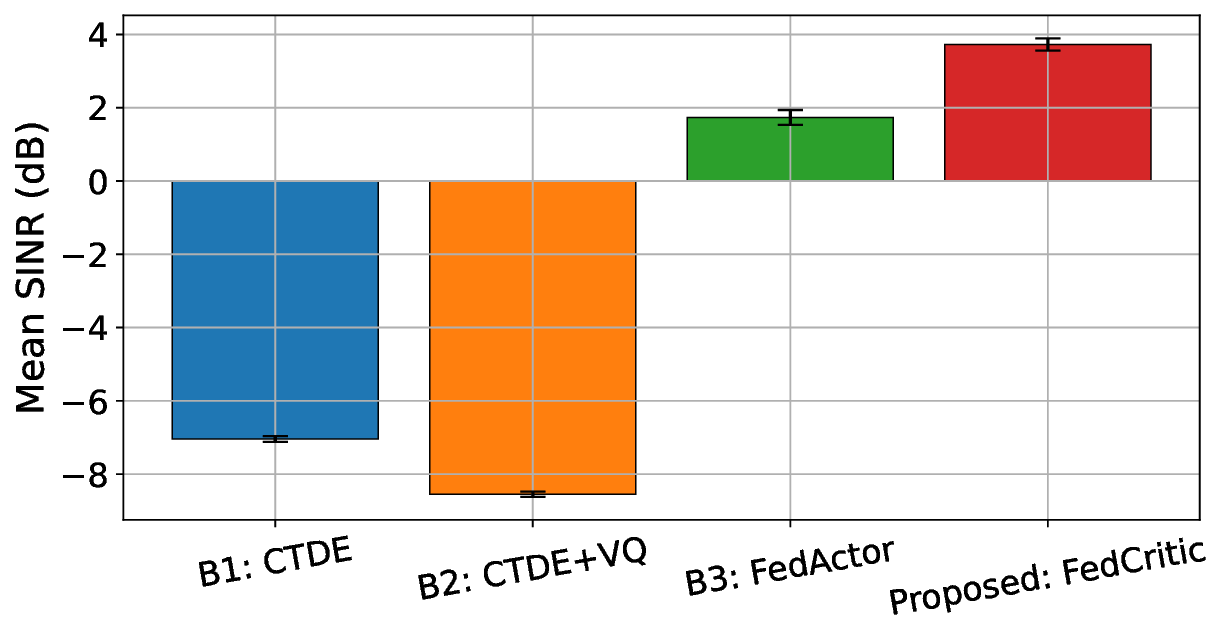}
    \caption{}
    \label{fig:sinr}
\end{subfigure}
\hfill
\begin{subfigure}[t]{0.49\linewidth}
    \centering
    \includegraphics[width=\linewidth]{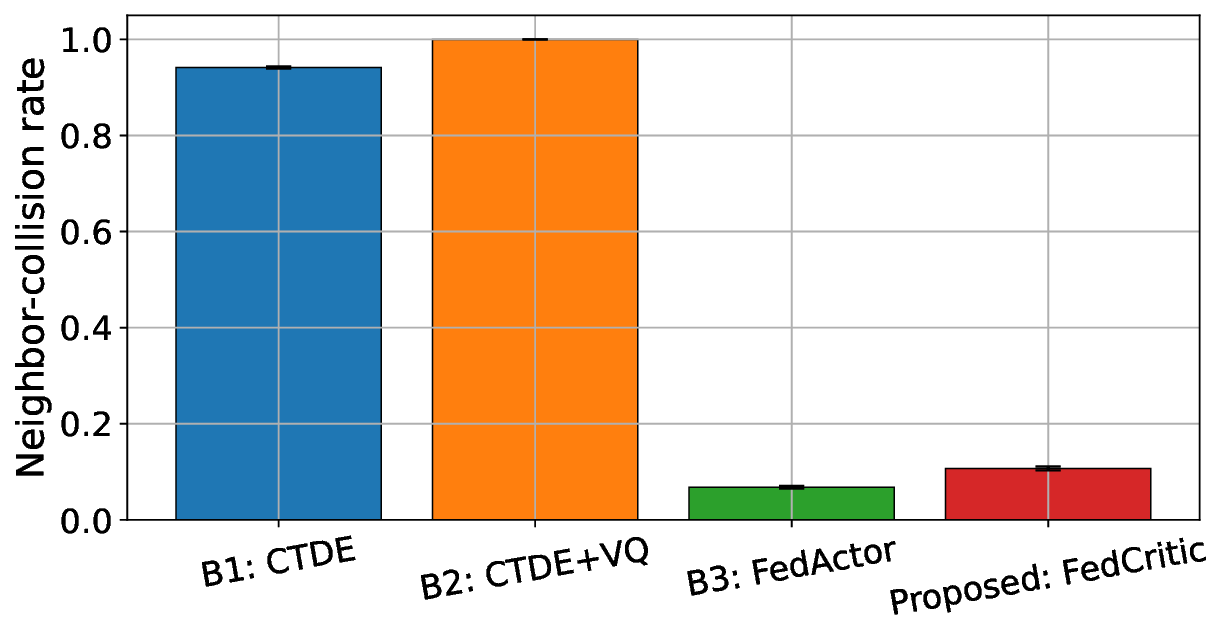}
    \caption{}
    \label{fig:coll}
\end{subfigure}
\caption{(a) Mean SINR and (b) neighbor-collision rate, over all active links across all BSs and subcarriers.}
\label{fig:ICI}
\end{figure}
\begin{figure}[t]
\centering
\begin{subfigure}[t]{0.49\linewidth}
    \centering
    \includegraphics[width=1.03\linewidth]{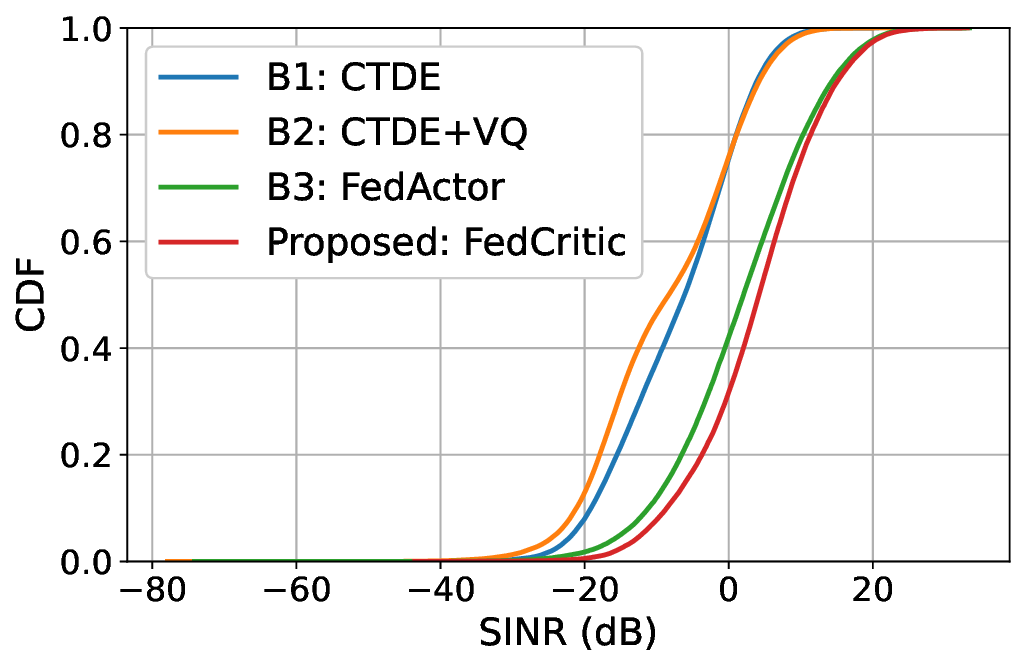}
    \caption{}
    \label{sinr_cdf}
\end{subfigure}
\hfill
\begin{subfigure}[t]{0.49\linewidth}
    \centering
    \includegraphics[width=1.03\linewidth]{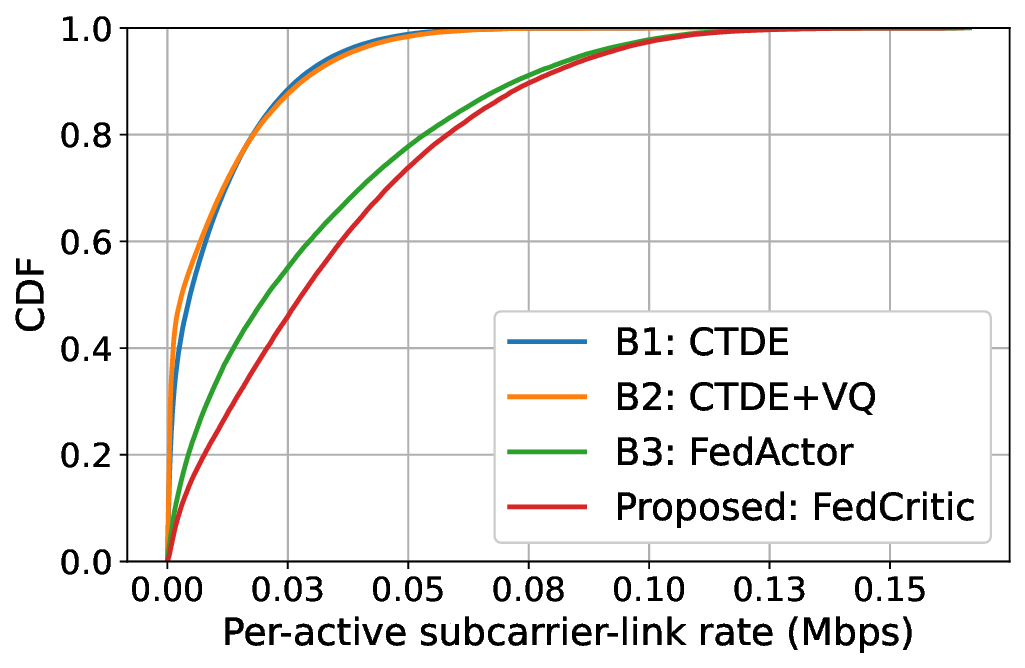}
    \caption{}
    \label{fig:cdf_active_link}
\end{subfigure}

\caption{(a) CDF of per-active subcarrier SINR (dB) and (b) CDF of per-active subcarrier-link rate (Mbps).}
\label{fig:CDF}
\end{figure}

\begin{figure}[t]
\centering

\begin{subfigure}[t]{0.5\linewidth}
    \centering
    \includegraphics[width=1.02\linewidth]{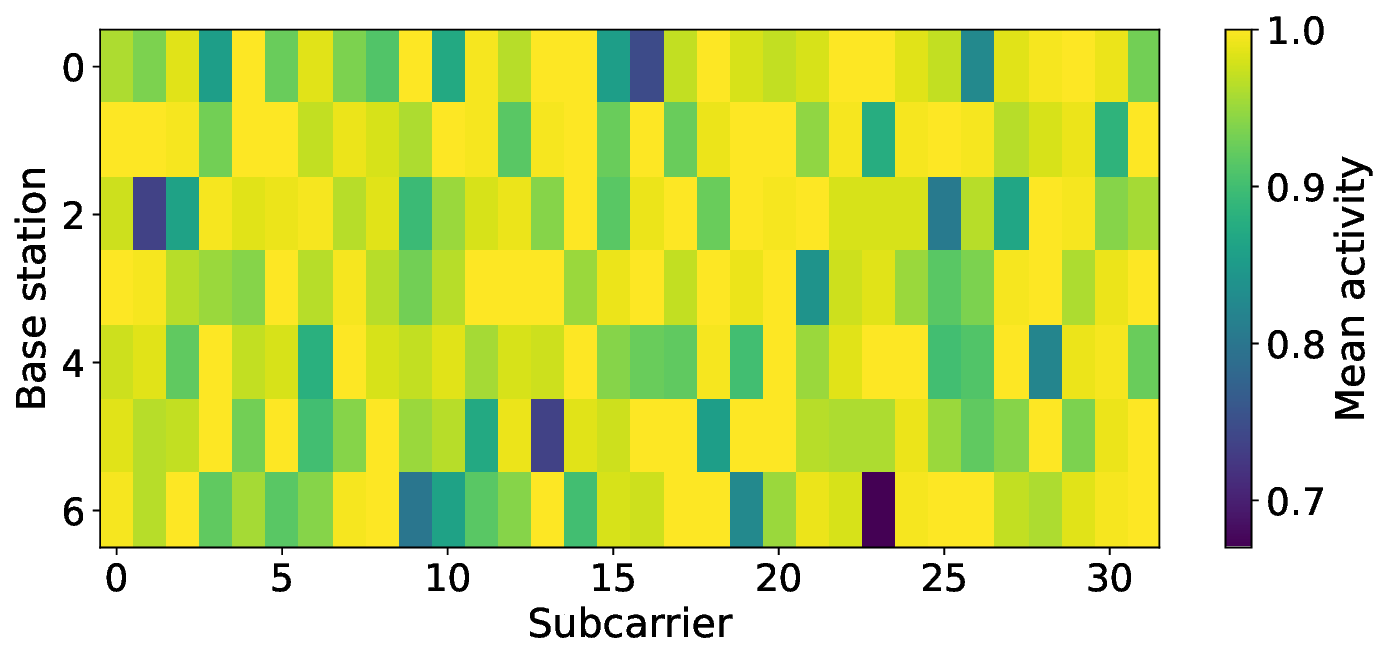}
    \caption{B2: CTDE + VQ.}
    \label{fig:activity_heatmap_b2}
\end{subfigure}
\hfill
\begin{subfigure}[t]{0.49\linewidth}
    \centering
    \includegraphics[width=1.02\linewidth]{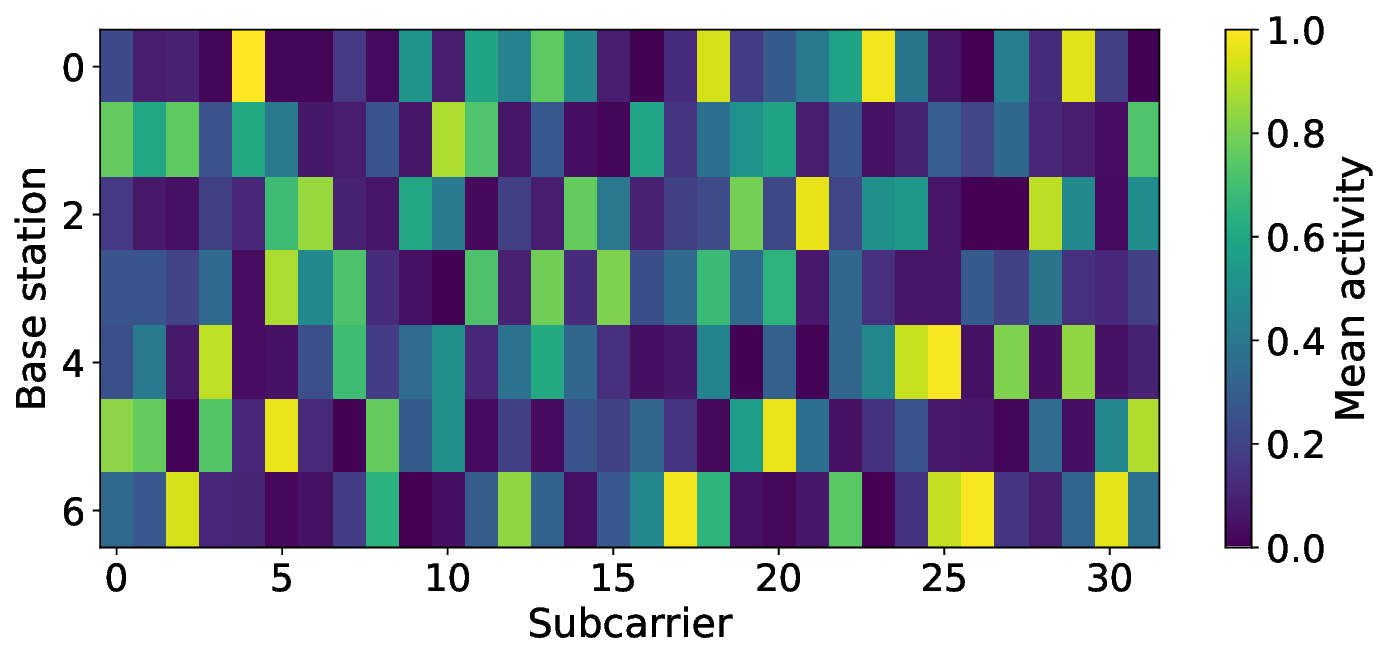}
    \caption{Proposed: FedCritic.}
    \label{fig:activity_heatmap_proposed}
\end{subfigure}

\caption{Activity (reuse intensity) heatmaps over BSs and subcarriers.}
\label{fig:activity_heatmap_compare}
\end{figure}

\section{Conclusion}\label{conclusion}
This paper studied interference-coupled reuse-1 multi-cell OFDMA resource management under practical information and coordination constraints. We proposed FedCritic, a fully decentralized MARL framework that removes the centralized critic bottleneck through neighbor-only gossip averaging while keeping execution local. Simulations showed that FedCritic improves SINR and sum-rate distributions, increases average sum-rate and fairness over CTDE baselines and greedy heuristics, and converges more stably with lower coordination overhead. Future work will extend the framework to larger and time-varying interference graphs, heterogeneous settings, richer PHY models, and asynchronous communication-efficient aggregation.
\enlargethispage{-2.3\baselineskip}
\vspace{-3mm}
\appendices
\renewcommand{\theequation}{\thesection.\arabic{equation}}
\section{Proof of Theorem~\ref{thm:critic_consensus} }
\label{app:proof_thm1}
\setcounter{equation}{0}

\begin{proof}
Let $\psi^s=[(\psi_1^s)^\top,\ldots,(\psi_N^s)^\top]^\top$ and
$\bar\psi^s=\frac{1}{N}\sum_{n=1}^N\psi_n^s$. Define the disagreement vector
$\delta^s \triangleq \psi^s-(\mathbf{1}\otimes \bar\psi^s)$, so that
$\|\delta^s\|^2=\sum_{n=1}^N\|\psi_n^s-\bar\psi^s\|^2$.
\vspace{1mm}

\textbf{(1) Average is preserved by mixing:}
Since $W_s$ is doubly stochastic for all $s$ (it is either $W$ or $I_N$), we have
$\mathbf{1}^\top W_s=\mathbf{1}^\top$, hence the gossip step preserves the average:
$\bar\psi^{s+1}=\bar\psi^{s+\frac12}$. The averaged iterate evolves as
\vspace{-4mm}
\begin{equation}
\bar\psi^{s+1}=\bar\psi^{s}-\eta_s\,\bar g^s,\qquad
\bar g^s\triangleq \frac{1}{N}\sum_{n=1}^N g_n(\psi_n^s;\xi_n^s).
\label{eq:avg_update_sketch}
\vspace{-2mm}
\end{equation}
Taking conditional expectation and using \eqref{eq:unbiased_grad} yields
$\mathbb{E}[\bar g^s\mid\mathcal{F}_s]=\frac{1}{N}\sum_{n=1}^N \nabla F_n(\psi_n^s)$.
\vspace{1mm}

\textbf{(2) Disagreement contracts at gossip rounds:}
At gossip rounds, $W_s=W$ and, with $J=\frac{1}{N}\mathbf{1}\mathbf{1}^\top$,
\vspace{-2mm}
{\small
\begin{align}
&\delta^{s+1}=\big((W-J)\otimes I_d\big)\,\delta^{s+\frac12}
\nonumber\\&\Rightarrow
\|\delta^{s+1}\|\le \|W-J\|_2\,\|\delta^{s+\frac12}\|
= \sigma\,\|\delta^{s+\frac12}\|,
\label{eq:contract_sketch}
\vspace{-2mm}
\end{align}
}
where $\sigma=\|W-J\|_2<1$ \cite{Theorem1}.
Between gossip rounds ($W_s=I_N$), $\delta^{s+1}=\delta^{s+\frac12}$.

The local SGD step adds bounded noise to disagreement. From
$\psi_n^{s+\frac12}=\psi_n^s-\eta_s g_n(\psi_n^s;\xi_n^s)$ and Assumption~(A2),
there exists $C_\delta>0$ such that
\vspace{-2mm}
{\small
\begin{equation}
\hspace{7mm}\mathbb{E}\big[\|\delta^{s+\frac12}\|^2\mid\mathcal{F}_s\big]
\le
\|\delta^s\|^2 + C_\delta \eta_s^2.
\label{eq:delta_growth_sketch}
\vspace{-1mm}
\end{equation}
}
Combining \eqref{eq:contract_sketch}--\eqref{eq:delta_growth_sketch} over periodic mixing yields a
supermartingale recursion with contraction factor $\sigma^2<1$ at gossip rounds and additive term $O(\eta_s^2)$.
Since $\sum_s \eta_s^2<\infty$, we obtain
\vspace{-2mm}
{\small
\begin{equation}
\lim_{s\to\infty}\mathbb{E}\|\delta^s\|^2 = 0,
\vspace{-2mm}
\end{equation}
}
which proves the consensus claim.
\vspace{1mm}

\textbf{(3) Stationarity of the averaged critic:}
By $L$-smoothness of ${\small F(\psi)=\frac{1}{N}\sum_{n=1}^N F_n(\psi)}$ and \eqref{eq:avg_update_sketch},
\vspace{-4mm}
{\small
\begin{equation}
\mathbb{E}\!\left[F(\bar\psi^{s+1})\right]
\le
\mathbb{E}\!\left[F(\bar\psi^s)\right]
-\eta_s\,\mathbb{E}\Big\langle \nabla F(\bar\psi^s),\frac{1}{N}\sum_{n=1}^N \nabla F_n(\psi_n^s)\Big\rangle
+O(\eta_s^2).
\label{eq:descent_sketch}
\end{equation}
}
Decompose $\frac{1}{N}\sum_n \nabla F_n(\psi_n^s)=\nabla F(\bar\psi^s)+e^s$, where by Lipschitz gradients,
$\|e^s\|\le \frac{L}{\sqrt{N}}\|\delta^s\|$. Using Young's inequality,
\vspace{-2mm}
{\small
\begin{equation}
\mathbb{E}\!\left[F(\bar\psi^{s+1})\right]
\le
\mathbb{E}\!\left[F(\bar\psi^s)\right]
-\tfrac{\eta_s}{2}\mathbb{E}\|\nabla F(\bar\psi^s)\|^2
+O(\eta_s)\mathbb{E}\|\delta^s\|^2
+O(\eta_s^2).
\label{eq:telescoping_sketch}
\end{equation}
}
Summing over $s$ and using that $F$ is bounded below (A1), $\sum_s \eta_s^2<\infty$ (A3), and
$\mathbb{E}\|\delta^s\|^2\to 0$ from Step~(2), yields
$\sum_s \eta_s\,\mathbb{E}\|\nabla F(\bar\psi^s)\|^2<\infty$.
Since $\sum_s \eta_s=\infty$, it follows that
\(
\liminf_{s\to\infty}\mathbb{E}\|\nabla F(\bar\psi^s)\|^2=0,
\)
proving the stationarity claim.
\end{proof}
\enlargethispage{-2.3\baselineskip}

%

\end{document}